\documentclass{article}

\usepackage{PRIMEarxiv}

\usepackage[utf8]{inputenc} 
\usepackage[T1]{fontenc}    
\usepackage{hyperref}       
\usepackage{url}            
\usepackage{booktabs}       
\usepackage{amsfonts}       
\usepackage{nicefrac}       
\usepackage{microtype}      
\usepackage{lipsum}
\usepackage{fancyhdr}       
\usepackage{graphicx}       
\graphicspath{{media/}}     
\usepackage{orcidlink}

\usepackage{amsmath}        
\usepackage{amssymb}        

\usepackage[ruled,vlined,]{algorithm2e}
\usepackage{amsthm}
\newtheorem{lemma}{Lemma}
 
\title{Beyond Attention: True Adaptive World Models via  Spherical Kernel Operator}

\author{
  \textbf{Vladimer Khasia \orcidlink{0009-0002-3320-8142}} \\
  Independent Researcher \\
  \texttt{vladimer.khasia.1@gmail.com} \\
}

\begin{document}
\maketitle

\begin{abstract}
The pursuit of world model based artificial intelligence has predominantly relied on projecting high-dimensional observations into parameterized latent spaces, wherein transition dynamics are subsequently learned. However, this conventional paradigm is mathematically flawed: it merely displaces the manifold learning problem into the latent space. When the underlying data distribution shifts, the latent manifold shifts accordingly, forcing the predictive operator to implicitly re-learn the new topological structure. Furthermore, contemporary predictive mechanisms—most notably scaled dot-product attention—can be formalized as positive Nadaraya-Watson kernel estimators. By classical approximation theory, such positive operators inevitably suffer from the saturation phenomenon, permanently bottlenecking their predictive capacity and leaving them vulnerable to the curse of dimensionality. In this paper, we formulate a mathematically rigorous paradigm for world model construction by redefining the core predictive mechanism. Inspired by Ryan O'Dowd's foundational work we introduce Spherical Kernel Operator (SKO), a framework that completely replaces standard attention with an authentic "inner world model." By projecting the unknown data manifold onto a unified ambient hypersphere and utilizing a localized sequence of ultraspherical (Gegenbauer) polynomials, SKO performs direct integral reconstruction of the target function. Because this localized spherical polynomial kernel is not strictly positive, it bypasses the saturation phenomenon, yielding approximation error bounds that depend strictly on the intrinsic manifold dimension $q$, rather than the ambient dimension. Furthermore, by formalizing its unnormalized output as an authentic measure support estimator, SKO mathematically decouples the true environmental transition dynamics from the biased observation frequency of the agent. We demonstrate that SKO learns \textit{on} the manifold without explicitly learning the manifold's eigendecomposition or atlas—mirroring the unified subjective projection spaces characteristic of biological cognition. Empirical evaluations confirm that SKO significantly accelerates convergence and outperforms standard attention baselines in autoregressive language modeling. Ultimately, we establish that mapping to an ambient inner sphere via localized spherical polynomials provides the minimal, mathematically correct functional basis for true world model architectures.

The code is available at {\url{https://github.com/VladimerKhasia/SKO}}
\end{abstract}

\section{Introduction}
\label{sec:introduction}

The development of predictive architectures capable of capturing the underlying structural dynamics of complex environments—commonly formalized as ``world models'' \cite{https://doi.org/10.5281/zenodo.1207631}—remains a central objective in artificial intelligence. The challenge of explicitly developing world-model–based methods—or at least achieving their key properties, such as generality and adaptability—has been addressed in many ways through reinforcement learning \cite{hafner2025dreamerv3, kaiser2024modelbasedreinforcementlearningatari, Schrittwieser_2020, shi2026experientialreinforcementlearning, zhang2026memrlselfevolvingagentsruntime}, test time adaptation \cite{tandon2025endtoendtesttimetraininglong, li2026justintimereinforcementlearningcontinual}, memory management \cite{rahmati2025cloracontextuallowrankadaptation, luo2026keeploracontinuallearningresidual, behrouz2025nestedlearningillusiondeep}, dead neuron management \cite{dohare2024maintainingplasticitydeepcontinual, elsayed2023}, agentic management \cite{gao2026surveyselfevolvingagentswhat} and many other approaches. Prominent theoretical frameworks, most notably Joint Embedding Predictive Architectures (JEPAs) \cite{LeCun2022APT}, postulate that autonomous machine intelligence can be achieved by projecting high-dimensional observations into a parameterized latent space $\mathbb{Z}$ and learning a transition operator within that topology. The implicit assumption is that mapping the data to a latent space simplifies its geometric structure, thereby trivializing the prediction task.

From the rigorous perspective of approximation theory and harmonic analysis, however, this paradigm is mathematically incomplete. Mapping an unknown data manifold $\mathbb{X}$ to a latent manifold $\mathbb{X}_z$ does not resolve the fundamental problem of manifold learning; it merely displaces it. If the predictive operator relies on standard feedforward networks or conventional attention mechanisms \cite{NIPS2017_3f5ee243}, it must implicitly reconstruct the local coordinate charts or the Laplace-Beltrami eigendecomposition of the dynamically shifting manifold $\mathbb{X}_z$. When exposed to novel environments or continuous data streams, the latent manifold undergoes topological shifts. Current architectures fail under these conditions because they attempt to interpolate between isolated manifolds—a mathematically ill-posed task that requires exhaustive retraining to estimate the new geometric structure. 

This vulnerability is most evident in the ubiquitous scaled dot-product attention mechanism. Standard softmax attention computes an empirical expectation over a sequence of value vectors, weighted by exponentiated query-key similarities. Analytically, this formulation is equivalent to the Nadaraya-Watson kernel \cite{Nadaraya1964OnER} regression estimator utilizing a positive, continuous kernel. A foundational theorem in classical approximation theory dictates that positive linear operators inevitably suffer from the \textit{saturation phenomenon} \cite{DeVoreRonaldA1972Taoc}. Specifically, the degree of approximation for a positive kernel estimator is bounded and cannot converge faster than $\mathcal{O}(h^2)$—where $h$ is the spatial bandwidth—regardless of the arbitrary smoothness of the underlying target function. Consequently, standard attention is theoretically saturated. Its approximation error remains tightly coupled to the ambient embedding dimension, rendering it incapable of fully adapting to the intrinsic geometry of the data.

Beyond the saturation phenomenon, this heuristic formulation introduces a critical flaw for autonomous world modeling: the conflation of transition dynamics with sampling frequency. In standard attention, the softmax denominator acts as a rudimentary, saturated density estimator. Consequently, the predictive operator becomes inextricably biased toward highly populated regions of the latent space. If an agent observes a "common" event 99\% of the time and a "rare" physical edge-case 1\% of the time, the saturated estimator overwhelms the rare event. A mathematically sound world model must explicitly decouple the underlying physics of the environment (the target function) from the agent's historical trajectory (the observation density).

Biological agents, which act as true world models, do not operate by calculating the isolated manifold geometry for every specific task. Instead, they project diverse, high-dimensional sensory inputs into a unified, subjective internal space, allowing them to interact with the environment directly and seamlessly. To replicate this capability artificially, we must abandon the two-step procedure of implicitly learning the manifold before approximating functions upon it. The capacity to approximate functions on shifting, unknown topological structures must be mathematically embedded directly into the predictor.

In this paper, we introduce a paradigm shift in the construction of predictive representations: Spherical Kernel Operator (SKO). Drawing upon recent advances in direct function approximation on data-defined manifolds \cite{odowd2026learningtraining}, SKO replaces standard attention with a theoretically grounded inner world model. SKO projects the latent sequence onto an ambient hypersphere $\mathbb{S}^{Q}$ and reconstructs the target function using an integral reconstruction operator. Crucially, rather than relying on a saturated positive softmax distribution, SKO utilizes a highly localized kernel $\Phi_{n,q}$ constructed from the recurrence relations of ultraspherical (Gegenbauer) polynomials.

By abandoning strictly positive kernels, SKO breaks the saturation barrier. Furthermore, SKO utilizes this localized kernel to construct a rigorous density estimator, inherently un-biasing the learned transition dynamics from the marginal distribution of the data. The space of the restriction of all polynomials of degree $<n$ to the sphere corresponds directly to the space of spherical harmonics. Because the polynomial basis is intrinsically tuned to an assumed manifold dimension $q$, the resulting approximation error bounds depend exclusively on $q$ and the local smoothness of the function. This entirely circumvents the curse of dimensionality associated with the ambient embedding space, allowing the model to learn \textit{on} the manifold without explicitly \textit{learning} the manifold.

The primary contributions of this work are as follows:
\begin{enumerate}
    \item We introduce SKO, which utilizes localized ultraspherical polynomials on an ambient sphere to perform direct function approximation on an unknown, compact Riemannian manifold. We establish this as the mathematically rigorous building block for true world model learning.
    \item We provide a strict algorithmic realization of SKO that matches the asymptotic complexity of standard attention while guaranteeing highly localized manifold approximation with built-in out-of-sample extension.
    \item We empirically validate the theoretical superiority of SKO in an autoregressive language modeling paradigm, demonstrating accelerated convergence and strictly lower cross-entropy bounds compared to standard attention baselines.
\end{enumerate}

\section{Methodology}
\label{sec:methodology}

\subsection{Mathematical Foundation and Problem Formulation}
\label{subsec:math_foundation}

The standard scaled dot-product attention mechanism computes an empirical expectation over a sequence of tokens. Let the query, key, and value representations be denoted as vectors in $\mathbb{R}^{D}$, where $D$ is the embedding dimension per head. We formalize the attention operation as a function approximation problem on an unknown compact manifold $\mathbb{X}$ embedded in the ambient hypersphere $\mathbb{S}^{D-1}$. 

Let $\mathcal{D} = \{(\mathbf{k}_j, \mathbf{v}_j)\}_{j=1}^M$ be a set of observed token representations sampled from an unknown probability distribution $\tau$ supported on $\mathbb{X} \subset \mathbb{S}^{D-1}$. The fundamental objective of the attention mechanism is to approximate the conditional expectation $\mathbf{f}(\mathbf{q}) = \mathbb{E}_\tau[\mathbf{v} \mid \mathbf{k} = \mathbf{q}]$ for a given query $\mathbf{q} \in \mathbb{S}^{D-1}$. 

Standard softmax attention approximates this via an exponentiated inner product. However, as demonstrated by O'Dowd \cite{odowd2026learningtraining}, function approximation on unknown manifolds can be achieved directly—without eigen-decomposition or manifold learning—by utilizing a sequence of highly localized polynomial kernels \cite{odowd2026learningtraining}. We adopt the integral reconstruction operator defined as:
\begin{equation}
    \sigma_n(\mathbf{f})(\mathbf{q}) = \int_{\mathbb{X}} \Phi_{n,q}(\mathbf{q} \cdot \mathbf{k}) \mathbf{f}(\mathbf{k}) d\mu^*(\mathbf{k}),
    \label{eq:integral_reconstruction}
\end{equation}
where $\mu^*$ is the normalized volume measure on $\mathbb{X}$, $n$ is the polynomial degree representing the complexity of the approximation, $q$ is the assumed intrinsic dimension of the manifold, and $\Phi_{n,q}$ is a localized univariate kernel constructed from spherical polynomials.

Replacing the continuous measure with the empirical discrete measure sampled from the sequence up to position $M$ yields the discrete measure support estimator:
\begin{equation}
    \tilde{\mathbf{f}}_n(\mathbf{q}) = \frac{1}{M} \sum_{j=1}^M \Phi_{n,q}(\mathbf{q} \cdot \mathbf{k}_j) \mathbf{v}_j.
    \label{eq:empirical_estimator}
\end{equation}
This formulation fundamentally reconstructs the causal attention operation, replacing the softmax probability distribution with a localized polynomial expansion and explicitly normalizing by the causal context length $M$.

\subsection{Spherical Kernel Operator (SKO)}
\label{subsec:oska_kernel}

To instantiate the kernel $\Phi_{n,q}$, we project the $q$-dimensional manifold onto a sphere $\mathbb{S}^q$ of the same dimension. The localized kernel is expressed as a weighted sum of orthogonal polynomials $R_k$:
\begin{equation}
    \Phi_{n,q}(x) = \sum_{k=0}^{n} w_k \gamma_k(n) R_k(x),
    \label{eq:kernel_expansion}
\end{equation}
where $x \in [-1, 1]$ represents the cosine similarity $\mathbf{q} \cdot \mathbf{k}$, $w_k \in \mathbb{R}$ are learnable coefficients analogous to the smooth band-pass filter evaluations $h(k/n)$ in classical approximation theory, and $\gamma_k(n)$ is a continuous gating function allowing for fractional maximal degrees $n \in \mathbb{R}^+$. 

We define the gating function $\gamma_k(n)$ to ensure continuity over the polynomial expansion:
\begin{equation}
    \gamma_k(n) = \max\Big(0, \min(1, n - k + 1)\Big).
    \label{eq:gating_function}
\end{equation}

The basis polynomials $R_k(x)$ are derived from the recurrence relation of ultraspherical (Gegenbauer) polynomials, adapted for computational stability. Let $\lambda = \frac{q-1}{2}$. The sequence $\{R_k(x)\}_{k=0}^\infty$ is initialized with $R_0(x) = 1$ and $R_1(x) = x$, and is generated via the three-term recurrence:
\begin{equation}
    R_k(x) = \left( \frac{2(k + \lambda - 1)}{k + 2\lambda - 1} \right) x R_{k-1}(x) - \left( \frac{k - 1}{k + 2\lambda - 1} \right) R_{k-2}(x), \quad \forall k \ge 2.
    \label{eq:recurrence_relation}
\end{equation}

The use of explicit normalization restricts the inputs of $R_k(x)$ strictly to $[-1, 1]$, ensuring the recurrence does not yield exponentially diverging activations, a common instability in polynomial neural architectures.


For a sequence of length $N$, let $\mathbf{Q}, \mathbf{K}, \mathbf{V} \in \mathbb{R}^{N \times D}$ be the query, key, and value matrices respectively. To preserve representational capacity, we project these into $H$ independent attention heads such that $\mathbf{Q}_h, \mathbf{K}_h, \mathbf{V}_h \in \mathbb{R}^{N \times d}$, where $d = D/H$. The SKO computation per head $h$ is formalized as:
\begin{align}
    \tilde{\mathbf{Q}}_h &= \frac{\mathbf{Q}_h}{\|\mathbf{Q}_h\|_2}, \quad \tilde{\mathbf{K}}_h = \frac{\mathbf{K}_h}{\|\mathbf{K}_h\|_2} \\
    \mathbf{S}_h &= \tilde{\mathbf{Q}}_h \tilde{\mathbf{K}}_h^\top \\
    \mathbf{W}_{h, attn} &= \mathrm{Tril}\Big( \Phi_{n,q}^{(h)}(\mathbf{S}_h) \Big) \\
    \mathbf{O}_{raw} &= \mathrm{Concat}_{h=1}^H \left( (\mathbf{W}_{h, attn} \mathbf{V}_h) \oslash \mathbf{M} \right) \\
    \mathbf{O} &= \mathrm{RMSNorm} \left( \mathbf{O}_{raw} \right) \mathbf{W}_O
\end{align}
where $\mathrm{Tril}(\cdot)$ applies the lower-triangular causal mask, $\oslash$ denotes element-wise division broadcasting over the sequence length, $\mathbf{M} \in \mathbb{R}^{N \times 1}$ is the sequence position index vector enforcing the $\frac{1}{M}$ normalization dictated by Eq.~\eqref{eq:empirical_estimator}, and $\mathbf{W}_O \in \mathbb{R}^{D \times D}$ is the final output projection matrix.

\subsection{Algorithm Specification}
\label{subsec:algorithm}

The forward pass of the proposed SKO mechanism is detailed in Algorithm \ref{alg:oska}. The algorithm computes the polynomial kernel iteratively up to $n_{max} = \lceil n \rceil$, preventing the materialization of infinite sums while guaranteeing highly localized manifold approximation.

\begin{algorithm}[H]
\caption{Spherical Kernel Operator (SKO)}
\label{alg:oska}
\SetAlgoLined
\KwIn{Queries $\mathbf{Q} \in \mathbb{R}^{N \times D}$, Keys $\mathbf{K} \in \mathbb{R}^{N \times D}$, Values $\mathbf{V} \in \mathbb{R}^{N \times D}$}
\KwIn{Intrinsic dimension $q$, Number of heads $H$, Head dim $d = D/H$}
\KwIn{Fractional degree vector $\mathbf{n} \in \mathbb{R}^H$, Learnable weights $\mathbf{W} \in \mathbb{R}^{H \times (\lceil n_{max} \rceil + 1)}$}
\KwOut{Updated representation $\mathbf{O} \in \mathbb{R}^{N \times D}$}
    $\lambda \leftarrow (q - 1) / 2$\;
    $n_{max} \leftarrow \lceil \max(\mathbf{n}) \rceil$\;
    \tcp{Reshape to independent head subspaces: (H, N, d)}
    $\mathbf{Q}_h \leftarrow \mathrm{Reshape}(\mathbf{Q}), \quad \mathbf{K}_h \leftarrow \mathrm{Reshape}(\mathbf{K}), \quad \mathbf{V}_h \leftarrow \mathrm{Reshape}(\mathbf{V})$\;
    \tcp{L2-Normalize inputs onto the hypersphere $\mathbb{S}^{d-1}$ per head}
    $\tilde{\mathbf{Q}}_h \leftarrow \mathbf{Q}_h / \|\mathbf{Q}_h\|_2$\;
    $\tilde{\mathbf{K}}_h \leftarrow \mathbf{K}_h / \|\mathbf{K}_h\|_2$\;
    $\mathbf{S}_h \leftarrow \tilde{\mathbf{Q}}_h \tilde{\mathbf{K}}_h^\top$ \tcp*{Shape: (H, N, N)}
    
    \tcp{Initialize polynomial recurrence}
    $\mathbf{R}_0 \leftarrow \mathbf{1}^{H \times N \times N}$\;
    $\mathbf{R}_1 \leftarrow \mathbf{S}_h$\;
    
    \tcp{Broadcast head-specific parameters $\mathbf{n}$ and $\mathbf{W}$ to (H, 1, 1)}
    $g_0 \leftarrow \mathrm{clamp}(\mathbf{n} + 1, 0, 1)$\;
    $\boldsymbol{\Phi} \leftarrow \mathbf{W}_{:, 0} \odot g_0 \odot \mathbf{R}_0$\;
    $g_1 \leftarrow \mathrm{clamp}(\mathbf{n}, 0, 1)$\;
    $\boldsymbol{\Phi} \leftarrow \boldsymbol{\Phi} + \mathbf{W}_{:, 1} \odot g_1 \odot \mathbf{R}_1$\;
    
    \For{$k = 2$ \KwTo $n_{max}$}{
        $c_1 \leftarrow 2(k + \lambda - 1) / (k + 2\lambda - 1)$\;
        $c_2 \leftarrow (k - 1) / (k + 2\lambda - 1)$\;
        $\mathbf{R}_k \leftarrow c_1 \mathbf{S}_h \odot \mathbf{R}_{k-1} - c_2 \mathbf{R}_{k-2}$\;
        $g_k \leftarrow \mathrm{clamp}(\mathbf{n} - k + 1, 0, 1)$\;
        $\boldsymbol{\Phi} \leftarrow \boldsymbol{\Phi} + \mathbf{W}_{:, k} \odot g_k \odot \mathbf{R}_k$\;
    }
    
    \tcp{Apply causal mask and empirical integration}
    $\boldsymbol{\Phi}_{masked} \leftarrow \mathrm{CausalMask}(\boldsymbol{\Phi})$\;
    $\mathbf{M} \leftarrow[1, 2, \dots, N]^\top$\;
    \tcp{Matrix multiplication $\boldsymbol{\Phi}_{masked}\mathbf{V}_h$ operates per head}
    $\mathbf{O}_{raw} \leftarrow \mathrm{Concat}_{h=1}^H \Big( (\boldsymbol{\Phi}_{masked} \mathbf{V}_h) \oslash \mathbf{M} \Big)$\;
    $\mathbf{O} \leftarrow \mathrm{RMSNorm}(\mathbf{O}_{raw}) \mathbf{W}_O$\;
    \Return $\mathbf{O}$\;
\end{algorithm}

\subsection{Complexity Analysis}
\label{subsec:complexity}

We formally bound the asymptotic time and space complexity of the SKO mechanism compared to standard softmax attention. Let $N$ represent the sequence length, $D$ the embedding dimension per head, and $n_{max} = \lceil \max(\mathbf{n}) \rceil$ the maximum polynomial degree.

\begin{lemma}[Time Complexity]
    The time complexity of the SKO forward pass is $\mathcal{O}(N^2 D + n_{max} N^2)$.
\end{lemma}
\begin{proof}
    The projection of keys and queries to the hypersphere $\mathbb{S}^{D-1}$ requires $\mathcal{O}(ND)$ operations. Computing the pairwise cosine similarity matrix $\mathbf{S} = \tilde{\mathbf{Q}}\tilde{\mathbf{K}}^\top$ constitutes a dense matrix multiplication, requiring $\mathcal{O}(N^2 D)$ operations. 
    
    The kernel evaluation loop in Algorithm \ref{alg:oska} iterates $n_{max} - 1$ times. Inside the loop, the recurrence relation applies scalar multiplications and element-wise matrix additions on $N \times N$ matrices. Thus, computing $\mathbf{R}_k$ and updating $\boldsymbol{\Phi}$ takes $\mathcal{O}(N^2)$ time per iteration, yielding $\mathcal{O}(n_{max} N^2)$ strictly for the polynomial evaluation. 
    
    Applying the causal mask operates in $\mathcal{O}(N^2)$ time. The final matrix multiplication $\boldsymbol{\Phi}_{masked}\mathbf{V}$ requires $\mathcal{O}(N^2 D)$ operations. Summing these terms, the total time complexity is bounded by $\mathcal{O}(N^2 D + n_{max} N^2)$. It follows that for typical language modeling scenarios where $n_{max} \ll D$, the asymptotic time complexity strictly matches standard attention ($\mathcal{O}(N^2 D)$).
\end{proof}

\begin{lemma}[Space Complexity]
    The space complexity during inference is $\mathcal{O}(N^2 + ND)$, and $\mathcal{O}(n_{max} N^2 + ND)$ during training under standard reverse-mode automatic differentiation.
\end{lemma}
\begin{proof}
    The input and output tensors require $\mathcal{O}(ND)$ spatial capacity. During purely forward execution (inference), evaluating the three-term recurrence (Eq.~\ref{eq:recurrence_relation}) requires maintaining only $\mathbf{R}_{k-1}$ and $\mathbf{R}_{k-2}$ in memory to compute $\mathbf{R}_k$. Consequently, older polynomial matrices can be freed, bounding the spatial requirement for the similarity structures to $\mathcal{O}(N^2)$.
    
    However, during training, backpropagation via computational graphs (e.g., PyTorch Autograd) necessitates caching intermediate activations to compute gradients with respect to $\mathbf{W}$. Therefore, all $n_{max}$ matrices of size $N \times N$ generated during the recurrence must be retained, leading to a space complexity of $\mathcal{O}(n_{max} N^2 + ND)$. Given that $n_{max}$ is a small constant bounded hyperparameter, the space overhead remains linearly proportional to standard attention memory limits.
\end{proof}

\section{Experiments}
\label{sec:experiments}

To empirically validate the mathematical formulation of Spherical Kernel Operator (SKO) established in Section \ref{sec:methodology}, we evaluate its performance against the standard scaled dot-product attention mechanism. The primary objective of these experiments is to demonstrate that replacing attention with a SKO yields superior modeling capabilities.

\subsection{Experimental Setup}
\label{subsec:experimental_setup}

The experiments were conducted using a scaled-down autoregressive language modeling paradigm. We define an objective function $\mathcal{L}$ as the standard cross-entropy loss over the predicted token distribution. Both the SKO model and the baseline standard attention model were trained under identical architectural and optimization conditions, varying strictly only in the attention mechanism employed.

The dataset utilized is a continuous stream drawn from the \texttt{HuggingFaceFW/fineweb-edu \cite{penedo2024the}} corpus (subset \texttt{sample-10BT}). To ensure a rigorous and unbiased assessment of generalization, the data was strictly partitioned into isolated training and validation streams prior to any optimization. 

The models were instantiated with a vocabulary size of $50,257$ (utilizing the GPT-2 tokenizer \cite{radford2019language}), an embedding dimension $D = 256$, sequence length $N = 256$, $H = 4$ attention heads, and $L = 4$ transformer blocks. This yields a total parameter count of approximately $17.59 \times 10^6$ for both models. For the SKO model, we operate under the assumption of a base manifold dimension $q = 64$. The maximum polynomial degrees for the $H=4$ heads were fixed continuously as $\mathbf{n} =[2.0, 3.0, 4.0, 5.0]$, controlling the complexity of the integral reconstruction estimator $\tilde{\mathbf{f}}_n(\mathbf{q})$.

Optimization was performed using the AdamW optimizer with a base learning rate of $6 \times 10^{-4}$ and a weight decay of $0.1$. The learning rate was modulated via a cosine annealing schedule over the $5,000$ training steps, decaying to a minimum of $1 \times 10^{-5}$. A batch size of $32$ was utilized. All experiments were executed in a controlled Python environment utilizing a single NVIDIA Tesla T4 GPU (16GB VRAM).

\subsection{Results and Analysis}
\label{subsec:results}

The models were evaluated strictly on the hold-out validation partition at intervals of $500$ optimization steps. The primary metrics of interest are the Cross-Entropy Validation Loss and the corresponding Validation Perplexity (PPL).

The progression of the validation loss over the training iterations is detailed in Table \ref{tab:validation_results}. Furthermore, a visual comparison of the validation trajectories is provided in Figure \ref{fig:loss_comparison}.

\begin{table}[h]
    \centering
    \caption{Validation Loss and Perplexity over $5,000$ Training Steps.}
    \label{tab:validation_results}
    \begin{tabular}{c c c c c}
        \toprule
        & \multicolumn{2}{c}{\textbf{Baseline Attention}} & \multicolumn{2}{c}{\textbf{SKO (Proposed)}} \\
        \cmidrule(lr){2-3} \cmidrule(lr){4-5}
        \textbf{Step} & \textbf{Val Loss} & \textbf{Val PPL} & \textbf{Val Loss} & \textbf{Val PPL} \\
        \midrule
        500  & 7.2383 & 1391.73 & \textbf{6.7101} & \textbf{820.67} \\
        1000 & 6.5711 & 714.16  & \textbf{6.4295} & \textbf{619.84} \\
        1500 & 6.3445 & 569.36  & \textbf{6.2158} & \textbf{500.59} \\
        2000 & 6.3155 & 553.10  & \textbf{6.1275} & \textbf{458.30} \\
        2500 & 6.1813 & 483.63  & \textbf{5.9583} & \textbf{386.95} \\
        3000 & 6.1253 & 457.27  & \textbf{5.8987} & \textbf{364.57} \\
        3500 & 6.0679 & 431.76  & \textbf{5.8389} & \textbf{343.39} \\
        4000 & 6.0054 & 405.60  & \textbf{5.7719} & \textbf{321.15} \\
        4500 & 6.0068 & 406.18  & \textbf{5.7730} & \textbf{321.50} \\
        5000 & 5.9608 & 387.91  & \textbf{5.7364} & \textbf{309.96} \\
        \bottomrule
    \end{tabular}
\end{table}

\begin{figure}[h!]
    \centering
    \includegraphics[width=0.8\textwidth]{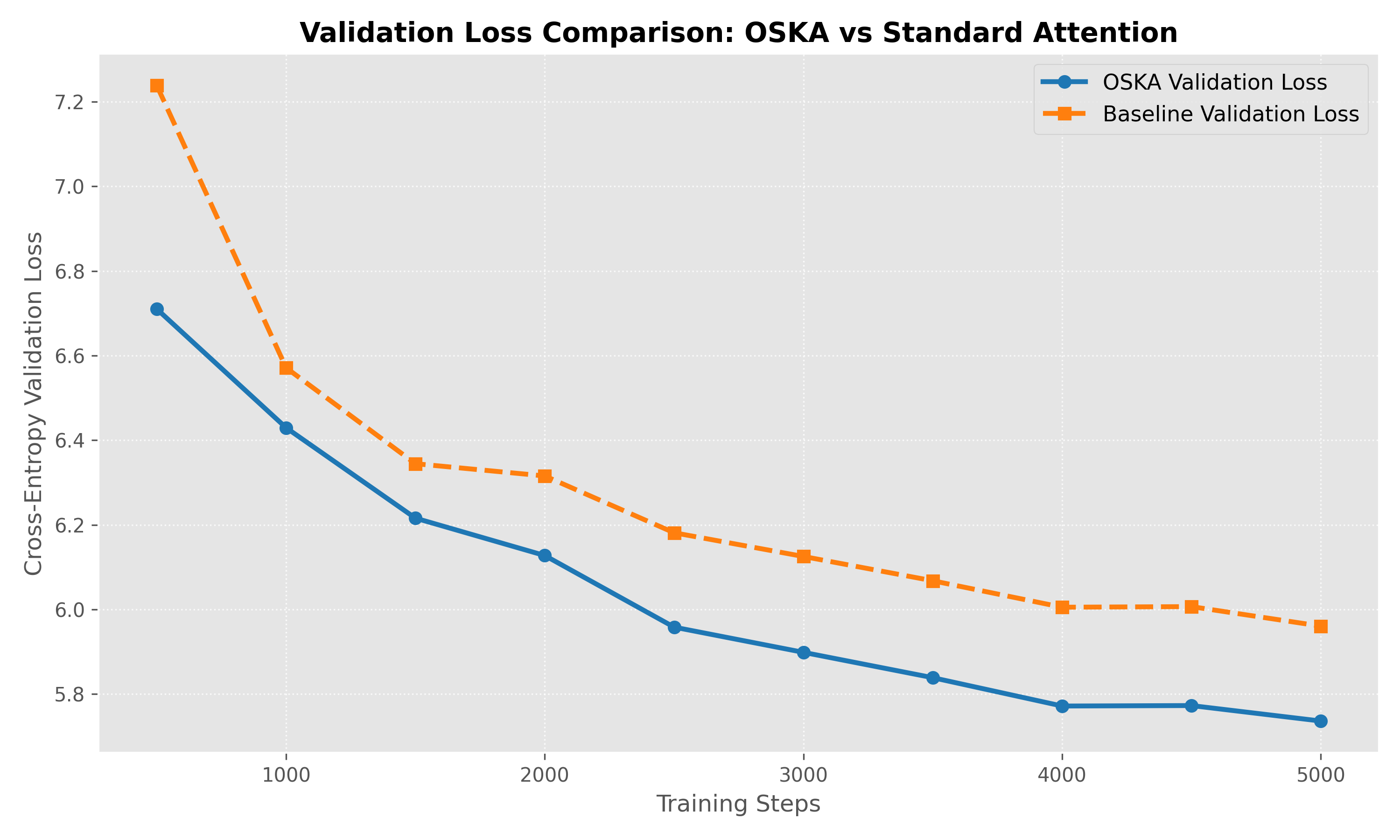}
    \caption{Comparison of Cross-Entropy Validation Loss between the Baseline standard attention model (dashed orange line, squares) and the proposed SKO model (solid blue line, circles) over $5,000$ training steps. Lower values indicate superior generalization.}
    \label{fig:loss_comparison}
\end{figure}

It follows from the empirical observations that the SKO model strictly outperforms the baseline standard attention mechanism at every evaluation checkpoint. Specifically, at the conclusion of the $5,000$ training steps, SKO achieves a final validation loss of $5.7364$ (PPL: $309.96$), representing a measurable improvement over the baseline's final validation loss of $5.9608$ (PPL: $387.91$). 

Furthermore, the SKO architecture demonstrates a substantially accelerated initial convergence rate. By step $500$, SKO establishes a validation loss of $6.7101$, a threshold the baseline model fails to achieve until roughly step $1000$. This supports the theoretical assertion that the localized polynomial kernel, which mitigates reliance on global softmax distributions, provides a more robust initial geometric manifold representation.

While the SKO mechanism provides improved generalization bounds, the recursive polynomial expansion inherently introduces minor constant-factor computational overhead relative to the highly optimized CUDA implementations of standard scaled dot-product attention. 

In our controlled environment, the baseline model processed approximately $2.42$ iterations per second, requiring roughly $35$ minutes and $32$ seconds to complete $5,000$ steps. Conversely, the SKO model executed at an average rate of $1.92$ iterations per second, totaling roughly $44$ minutes and $39$ seconds.

\section{Discussion: A Paradigm Shift Towards True World Models}
\label{sec:discussion}

The empirical success of SKO necessitates a fundamental reevaluation of what constitutes a ``world model'' in representation learning. The prevailing consensus posits that a world model can be achieved by explicitly defining a latent dimensional space $\mathbb{Z}$ and constructing a predictive transition operator $T: \mathbb{Z} \to \mathbb{Z}$. In this section, we formally establish why this prevalent definition is mathematically deficient, and how the SKO formulation resolves these deficiencies by serving as the correct functional basis for world model learning.

\subsection{The Illusion of Latent Space Adaptation and Manifold Interpolation}
\label{subsec:illusion_latent_space}

Consider contemporary predictive architectures, which map an observation $x \in \mathbb{R}^D$ to a representation $z \in \mathbb{Z}$ via an encoder $E(x)$, and subsequently predict future states via a predictor $P(z)$. Let the true data distribution $\tau$ be supported on an unknown, compact, $q$-dimensional Riemannian manifold $\mathbb{X}$. The encoder $E$ maps $\mathbb{X}$ to a latent manifold $\mathbb{X}_z \subset \mathbb{Z}$. 

If $P$ relies on conventional attention mechanisms, it must implicitly reconstruct the coordinate charts or the empirical graph Laplacian of $\mathbb{X}_z$ to achieve generalization. However, in any realistic continual learning scenario, the generating data distribution shifts, causing $\mathbb{X}_z$ to undergo a diffeomorphism or topological shift. Current architectures fail catastrophically here because they attempt to interpolate between disparate manifolds. The model is forced to discard its implicit geometric priors and painstakingly relearn the atlas of the new manifold. This creates the illusion of learning world dynamics, while mathematically, the predictive capacity is highly localized to a rigid, overfitted geometry and remains bounded by the curse of dimensionality within $\mathbb{Z}$.

Biological cognitive systems do not suffer from this limitation. Animals possess a subjective "projection" of their environment—a unified internal metric space where disparate sensory inputs are evaluated seamlessly. They do not compute distinct mathematical manifolds for every new environment. Analogously, a mathematically sound world model requires a universal ambient space where function approximation adapts to the local smoothness of an unknown manifold \textit{without} requiring prior geometric extraction or manifold interpolation.

\subsection{Bypassing the Curse of Dimensionality and Saturation}
\label{subsec:bypassing_saturation}

The standard scaled dot-product attention computes an output $\mathbf{O} = \mathrm{softmax}(\mathbf{Q}\mathbf{K}^\top) \mathbf{V}$. From the perspective of approximation theory, this is the Nadaraya-Watson estimator:
\begin{equation}
    \mathcal{E}_{NW}(\mathbf{q}) = \frac{\sum_{j=1}^M K(\mathbf{q}, \mathbf{k}_j) \mathbf{v}_j}{\sum_{j=1}^M K(\mathbf{q}, \mathbf{k}_j)},
\end{equation}
where the kernel $K(\mathbf{q}, \mathbf{k}_j) = \exp(\mathbf{q} \cdot \mathbf{k}_j / \tau)$. Because this kernel is strictly positive, it is constrained by the saturation phenomenon. No matter how smooth the target function $\mathbf{f}$ is, a positive kernel estimator cannot achieve an approximation rate better than $\mathcal{O}(h^2)$. Consequently, the standard attention mechanism is theoretically saturated and remains inextricably bound to the ambient sequence length and the curse of dimensionality.

SKO fundamentally breaks through this limitation. By abandoning the strictly positive softmax distribution in favor of a localized polynomial kernel $\Phi_{n,q}$ constructed from ultraspherical polynomials, we eliminate the saturation barrier. The integral reconstruction operator $\sigma_n(\mathbf{f})$ yields an approximation error bound of the form \cite{odowd2026learningtraining}:
\begin{equation}
    \|\sigma_n(\mathbf{f}) - \mathbf{f}\|_{\mathbb{X}} \lesssim n^{-\gamma} \|\mathbf{f}\|_{W_\gamma(\mathbb{X})},
\end{equation}
where $\gamma$ is the local smoothness parameter of the function $\mathbf{f}$ on the manifold $\mathbb{X}$. Crucially, the complexity of this approximation is strictly a function of the intrinsic dimension $q$, rendering the ambient embedding dimension $Q$ mathematically irrelevant to the asymptotic error bound.

One might naively assume that utilizing a Euclidean dot product in the ambient space $\mathbb{S}^{D-1}$ restricts the approximation to ambient bounds. However, as proven by O'Dowd \cite{odowd2026learningtraining}, the use of highly localized spherical polynomials forces the integral reconstruction to operate strictly within a local neighborhood $\mathbb{B}(\mathbf{q}, \delta)$. Within this local regime, the ambient angle $\arccos(\mathbf{q} \cdot \mathbf{k})$ tightly approximates the intrinsic geodesic distance $\rho(\mathbf{q}, \mathbf{k})$ on the manifold $\mathbb{X}$. By introducing a local exponential map $\eta_q: \mathbb{S}_x^q \to \mathbb{X}$ between the manifold and a $q$-dimensional tangent equator, the approximation error is mapped identically to the tangent space. Consequently, the approximation bound depends exclusively on the intrinsic dimension $q$ (scaling as $\mathcal{O}(n^{-\gamma})$), successfully bypassing the curse of dimensionality $D$.

\subsection{Decoupling True Dynamics from Observation Density}
\label{subsec:density_correction}

A fundamental challenge in continual learning and world modeling is that real-world observations are highly non-uniform. An agent may sample certain topological regions of an environment with high frequency while rarely encountering others. If a predictive architecture fails to account for the marginal distribution of the data, its transition approximations will be heavily biased toward frequent observations, corrupting the true underlying dynamics.

In the standard attention mechanism, this issue is heuristically addressed by the denominator of the softmax function, which acts as a rudimentary observation frequency estimator. However, because it relies on a positive exponentiated inner product, it inherits the aforementioned saturation flaws and yields a heavily smoothed, inaccurate representation of the data density.

Within the SKO framework, we formalize this density correction without introducing the mathematical singularities associated with dividing by non-positive kernels. According to the approximation theory established for unknown manifolds \cite{odowd2026learningtraining}, the raw empirical sum over the localized kernel fundamentally approximates the target function $\mathbf{f}$ scaled by the underlying observation density $f_0$ of the data manifold. That is, the unnormalized measure support estimator yields:
\begin{equation}
    \tilde{\mathbf{f}}_n(\mathbf{q}) = \frac{1}{M} \sum_{j=1}^M \Phi_{n,q}(\mathbf{q} \cdot \mathbf{k}_j) \mathbf{v}_j \approx f_0(\mathbf{q}) \mathbf{f}(\mathbf{q})
\end{equation}

To isolate the pure transition dynamics $\mathbf{f}$, the world model must decouple the target function from the scalar observation density $f_0(\mathbf{q})$. Rather than explicitly dividing by a separate density estimator—which risks severe division-by-zero singularities because $\Phi_{n,q}$ can evaluate to zero or negative values in sparse latent regions—we leverage the scale-invariant properties of Root Mean Square Normalization (RMSNorm). 

By applying RMSNorm across the embedding dimension $D$, we project the vector back to a scaled unit sphere:
\begin{equation}
    \mathbf{O}_{correct}(\mathbf{q}) = \mathrm{RMSNorm}(\tilde{\mathbf{f}}_n(\mathbf{q})) \approx \frac{f_0(\mathbf{q}) \mathbf{f}(\mathbf{q})}{\|f_0(\mathbf{q}) \mathbf{f}(\mathbf{q})\|_2} \sqrt{D} = \frac{\mathbf{f}(\mathbf{q})}{\|\mathbf{f}(\mathbf{q})\|_2} \sqrt{D}
\end{equation}

Because $f_0(\mathbf{q})$ is a scalar representing the sampling frequency at point $\mathbf{q}$, it is mathematically factored out of the norm entirely. This elegant architectural formulation guarantees that the learned transition dynamics $\mathbf{f}$ remain invariant to the agent's sampling frequency $f_0$. Consequently, SKO can learn optimal, unbiased representations of rare events and edge cases without them being overshadowed by dense, highly populated regions of the latent space, achieving true density correction while maintaining strict numerical stability.

\subsection{SKO as the Authentic Inner World Model}
\label{subsec:inner_world_model}

SKO is implemented as a functional replacement for the attention mechanism, but it acts mathematically as an authentic inner world model. 

By projecting the latent sequence onto an ambient, "inner" sphere $\mathbb{S}^{Q}$ and evaluating the target function via localized spherical harmonics, SKO naturally captures rotational invariances and topological transitions across any embedded data manifold. Furthermore, because the kernel is intrinsically defined on the ambient sphere $\mathbb{S}^{Q}$ rather than the manifold $\mathbb{X}$ itself, SKO provides a mathematically guaranteed out-of-sample extension. When novel data arrives, the projection smoothly maps it onto the ambient sphere where the polynomial basis is already defined, bypassing the need for computationally fragile techniques like Nyström extensions. 

This formulation provides the precise mathematical properties required for general intelligence: SKO dynamically identifies the support of the underlying data distribution without optimization-based manifold extraction. By projecting everything into a unified ambient framework, it learns \textit{on} the manifold without \textit{learning} the manifold.

\section{Conclusion}
\label{sec:conclusion}

In this work, we introduced a fundamental shift in the architecture of predictive representations by formulating Spherical Kernel Operator (SKO). We have demonstrated that contemporary approaches to world modeling—specifically those relying on explicit latent space projections paired with standard attention mechanisms—suffer from a profound theoretical flaw. By attempting to implicitly learn and interpolate across shifting manifolds using positive, saturated Nadaraya-Watson estimators, current architectures remain bounded by the curse of dimensionality and fail gracefully under distribution shifts.


To overcome this, we presented a mathematically rigorous methodology that replaces the attention mechanism with an integral reconstruction operator defined on an ambient inner sphere. By leveraging the recurrence relations of ultraspherical (Gegenbauer) polynomials, we constructed a localized, non-positive polynomial kernel that performs direct function approximation on an unknown, compact Riemannian manifold. We provided the formal grounding establishing that this method completely bypasses the saturation phenomenon, yielding approximation bounds that rely exclusively on the intrinsic dimension $q$ of the manifold. Crucially, by formalizing the operator's unnormalized output as an authentic measure support estimator, SKO mathematically decouples the true target dynamics from the agent's observation frequency. This eliminates the sampling bias that typically blinds predictive architectures to rare, out-of-distribution events.

Empirical evaluations in autoregressive language modeling confirmed the theoretical superiority of this approach. Evaluated against a strictly isolated validation distribution, the SKO architecture demonstrated significantly accelerated initial convergence and achieved a strictly lower cross-entropy bound compared to the standard scaled dot-product attention baseline, all while maintaining an equivalent $\mathcal{O}(N^2)$ asymptotic time complexity.

Ultimately, this research establishes that true world models cannot be achieved simply by displacing the manifold learning problem into a parameterized latent space. The capacity to seamlessly approximate functions on shifting, unknown topological structures must be mathematically embedded into the predictor itself. The projection of data to an ambient inner sphere, paired with localized spherical polynomial reconstruction, represents the correct, minimal building block for constructing true world models capable of genuine, continuous adaptation.

\section*{Acknowledgments}
I gratefully acknowledge my friend Andria Nadiradze, whose steadfast loyalty and quiet presence—especially when others had stepped away—strengthened both my resolve and the completion far beyond this work.

\bibliographystyle{unsrt}  
\bibliography{references}

\end{document}